\documentclass[final,12pt]{colt2019} % Anonymized submission
\title[On the Performance of Thompson Sampling on Logistic Bandits]{On the Performance of Thompson Sampling on Logistic Bandits}
\usepackage{times}
\usepackage{shi_dong_template}
\usepackage[stable]{footmisc}
\usepackage{changepage}

\def\shownotes{1}  %set 1 to show author notes
\ifnum\shownotes=1

\coltauthor{%
 \Name{Shi Dong} \Email{sdong15@stanford.edu}\\
 \Name{Tengyu Ma} \Email{tengyuma@stanford.edu}\\
 \Name{Benjamin {Van Roy}} \Email{bvr@stanford.edu}\\
 \addr Stanford University
}

\def\A{\mathcal{A}}
\def\B{\mathrm{B}}
\def\D{\lambda}

\def\hist{\mathcal{H}}
\def\hX{\hat{X}}
\def\hY{\hat{Y}}
\def\hU{\hat{U}}

\def\tY{\tilde{Y}}

\def\Var{\mathrm{Var}}
\def\ind{\mathbb{I}}
\def\phib{\phi_\beta}

\def\gbl{\gamma_{\beta,\lambda}}
\def\gbbl{\bar{\gamma}_{\beta,\lambda}}
\def\gby{\gamma_{\beta,\sigma(x)}}
\def\gbby{\bar{\gamma}_{\beta,\sigma(x)}}
\def\gbY{\gamma_{\beta,\sigma(X)}}
\def\gbbY{\bar{\gamma}_{\beta,\sigma(X)}}
\def\nb{\nu_\beta}
\def\TS{\mathrm{TS}}
\def\zbl{z_{\beta,\lambda}}

\def\wbl{w_{\beta,\lambda}}

\def\wbY{w_{\beta, \sigma(X)}}
\renewcommand{\thefootnote}{\arabic{footnote}}
\begin{document}

\maketitle
\footnote{Accepted for presentation at the Conference on Learning Theory (COLT) 2019.}
\def\FOOTNOTE{In defining ``log-odds,'' we use base $e^\beta$ rather than $e$.  As a result, the ``log-odds'' throughout this article refers to $a^\T\theta$ instead of $\beta a^\T\theta$.}
\begin{abstract}%
We study the logistic bandit, in which rewards are binary with success probability $\exp(\beta a^\top \theta) / (1 + \exp(\beta a^\top \theta))$ and actions $a$ and coefficients $\theta$ are within the $d$-dimensional unit ball.  While prior regret bounds for algorithms that address the logistic bandit exhibit exponential dependence on the slope parameter $\b$, we establish a regret bound for Thompson sampling that is independent of $\b$.  Specifically, we establish that, when the set of feasible actions is identical to the set of possible coefficient vectors, the Bayesian regret of Thompson sampling is $\t{O}(d\sqrt{T})$.  We also establish a $\t{O}(\sqrt{d\eta T}/\D)$ bound that applies more broadly, where $\D$ is the worst-case optimal log-odds\footnote{\FOOTNOTE} and $\eta$ is the ``fragility dimension,'' a new statistic we define to capture the degree to which an optimal action for one model fails to satisfice for others.  We demonstrate that the fragility dimension plays an essential role by showing that, for any $\e > 0$, no algorithm can achieve $\mathrm{poly}(d, 1/\D)\cdot T^{1-\e}$ regret.
\end{abstract}


\section{Introduction}
In the \textit{logistic bandit} an agent observes a binary reward after each action, with outcome probabilities governed by 
a logistic function:
$$\mathbb{P}\left(\text{reward}=1 \Big| \text{action} = a\right) = \frac{e^{\b a^\top \theta}}{1+e^{\b a ^\top \theta}}.$$
Each action $a$ and parameter vector $\theta$ is a vector within the $d$-dimensional unit ball.  The agent initially knows the
scale parameter $\beta$ but is uncertain about the coefficient vector $\theta$.  The problem of learning to improve action selection over repeated interactions is sometimes referred to as the \textit{logistic bandit problem} or \textit{online logistic regression}.

The logistic bandit serves as a model for a wide range of applications. One example is the problem of personalized recommendation, in which a service provider successively recommends content, receiving only binary responses from users, indicating ``like'' or ``dislike.''
A growing literature treats the design and analysis of action selection algorithms for the logistic bandit.  Upper-confidence-bound (UCB) algorithms have been analyzed in \cite{filippi2010parametric,li2017provable,russo2013eluder}, while Thompson sampling (\cite{thompson1933likelihood}) was treated in \cite{russo2014posterior} and \cite{abeille2017linear}.  Each of these algorithms has been shown to converge on the optimal action with time dependence $\t{O}(1/\sqrt{T})$, where $\t{O}$ ignores poly-logarithmic factors.  However, previous analyses leave open the possibility that the convergence time increases exponentially with the parameter $\b$, which seems counterintuitive.  In particular, as $\b$ increases, distinctions between good and bad actions become more definitive, which should make them easier to learn.\par
To shed light on this issue, we build on an information-theoretic line of analysis, which was first proposed in \cite{russo2016information} and further developed in \cite{bubeck2016multi} and \cite{dong2018information}.  A critical device here is the \textit{information ratio}, which quantifies the one-stage trade-off between exploration and exploitation.  The information ratio has also motivated the design of efficient bandit algorithms, as in \cite{russo2014learning}, \cite{russo2018satisficing} and \cite{liu2018information}. While prior bounds on the information ratio pertain only to independent or linear bandits, in this work we develop a new technique for bounding the information ratio of a logistic bandit.  This leads to a stronger regret bound and insight into the role of $\b$.
\begin{center}
\begin{table}
\begin{adjustwidth}{-0.7cm}{}
\begin{tabular}{ |c||c|c| } 
 \hline
 Algorithm & Regret Upper Bound & Notes \\ 
 \hline
 \begin{tabular}{c} GLM-UCB\\ (\cite{filippi2010parametric})\end{tabular} & $O\left(e^\b\cdot d\cdot T^{1/2}\log^{3/2}T \right)$ & Frequentist bound. \\ 
 \hline
 \begin{tabular}{c} A variation of GLM-UCB\\(\cite{russo2013eluder})\end{tabular} & $O\left(e^\b \log \b\cdot d\cdot T^{1/2}\right)$ & Bayesian bound. \\ 
 \hline
 \begin{tabular}{c} SupCB-GLM\\ (\cite{li2017provable})\end{tabular} & $O\left(e^\b\cdot (d \log K)^{1/2}\cdot T^{1/2}\log T \right)$ & \begin{tabular}{c} Frequentist bound,\\ $K$ is the number of actions. \end{tabular} \\
 \hline
 \begin{tabular}{c} Thompson Sampling\\ (\cite{russo2014posterior})\end{tabular} & $O\left(e^\b\cdot d \cdot T^{1/2}\log^{3/2} T \right)$ & Bayesian bound.\\
 \hline
 \begin{tabular}{c} Thompson Sampling\\ (\cite{abeille2017linear})\end{tabular} & $O\left(e^\b\cdot d^{3/2}\log^{1/2}d \cdot T^{1/2}\log^{3/2} T \right)$ & Frequentist bound.\\
 \hline
 \begin{tabular}{c} {\bf Thompson Sampling}\\ {\bf (this work)}\end{tabular} & $O\left(\l^{-1}\cdot \big( d(\eta\lor d) \big)^{1/2} \cdot T^{1/2}\log^{1/2} T \right)$ & \begin{tabular}{c} Bayesian bound,\\ $\l$ and $\eta$ are independent of $\beta$\\ (defined in Section \ref{sec:main-results}).\end{tabular}\\
 \hline
\end{tabular}
\end{adjustwidth}
	\caption{Comparison of various results on logistic bandits. The upper bound in this work depends on $\b$-independent parameters $\l$ and $\eta$, defined in Assumption \ref{assp:main} and Definition \ref{def:DOI}, respectively. We use the notation $a\lor b=\max\{a,b\}$.}
\end{table}
\end{center}

\noindent\textbf{Our Contributions.}  Let $\A$ and $\Theta$ be the set of feasible actions and the support of $\theta$, respectively.
Under an assumption that $\A = \Theta$, we establish a $\t{O}(d \sqrt{T})$ bound on Bayesian regret.  This bound scales with the dimension $d$, but notably exhibits no dependence on $\beta$ or the number of feasible actions.  We then generalize this bound, relaxing the assumption that $\A = \Theta$ while introducing dependence on two statistics of the these sets: the {\it worst-case optimal log-odds} $\lambda = \min_{\theta \in \Theta} \max_{a \in \A} \a^\top \theta$ and the {\it fragility dimension} $\eta$, which is the number of possible models such that the optimal action for each yields success probability no greater than 50\% for any other.  Assuming $\lambda > 0$, we establish a $\t{O}(\sqrt{d \eta T} / \lambda)$ bound on Bayesian regret.  We also demonstrate that the fragility dimension plays an essential role, as for any function $f$, polynomial $p$, and $\e>0$, any algorithm for the logistic bandit cannot achieve Bayesian regret uniformly bounded by $f(\D)p(d)T^{1-\e}$.  We believe that, although $\eta$ can grow exponentially with $d$, in most relevant contexts $\eta$ should scale at most linearly with $d$.

The assumption that the worst-case optimal log-odds are positive may be restrictive.  This is equivalent to assuming that the for each possible model, the optimal action yields more than 50\% probability of success.  However, this assumption is essential, since it ensures that the fragility dimension is well-defined.  When the worst-case optimal log-odds are negative, the geometry of action and parameter sets plays a less significant role than parameter $\b$, therefore we conjecture that the exponential dependence on $\b$ is inevitable.  This could be an interesting direction for future research.\\

\noindent\textbf{Notations.} Throughout this article, for integer $n$ we will use $[n]$ to denote the set $\{1,\dots,n\}$. We will also use $\BF{B}_d$ and $\BF{S}_{d-1}$ to denote the unit ball and the unit sphere in $\BB{R}^d$, respectively. 

\section{Problem Settings}
We consider \textit{Bayesian generalized linear bandits}, defined as a tuple $\C{L} = (\A, \Th, R, \phi, \rho)$, where $\C{A}$ and $\Th$ are the action and parameter set, respectively, $R$ is a stochastic process representing the reward of playing each action, $\phi$ is the link function, and $\rho$ is the prior distribution over $\Th$, which represents our prior belief of the groundtruth parameter $\th^*$. Throughout this article, to avoid measure-theoretic subtleties, we assume that both $\A$ and $\Th$ are finite subsets of $\BF{B}_d$. For simplicity, we assume that there exists a one-to-one mapping\footnote{Note that Thompson sampling does not consider actions that are not optimal for any parameter. If an action is optimal for multiple parameters, we can add identical copies of the action to the action set such that the mapping between each parameter and the corresponding optimal action is one-to-one.} between each parameter and the corresponding optimal action. Specifically, let $\C{A}=\{a^1, \dots, a^N\}$ and $\Th=\{\th^1, \dots, \th^N\}$, with
\[
     \argmax_{a\in\C{A}} \E \left[R(a) | \th^* = \th^i\right] = \{a_i\},\quad \forall i=1,\dots, N.
\]
To specify the one-to-one mapping, for each $\th\in\Th$ we define $\a(\th)$ to be the unique action that maximizes $\E[R(a)|\th^*=\th]$. Letting $A^*$ be the optimal action, which is a random variable under our Bayesian setting, naturally we have $A^* = \a(\th^*)$. \par 
The reward $R$ is related to the inner product between the action and the parameter by the link function $\phi$, as
\[
\E \left[ R(a) | \th^* = \th \right] = \phi(a^\T\th),\quad \forall a\in\C{A}, \th\in\Th.
\]
Specifically, in logistic bandits, the reward $R$ is the binary process $R_\B$ and the link function is given by
\[
\phib(x) = \frac{e^{\b x}}{1 + e^{\b x}},
\]
where $\b>0$ is a parameter that characterizes the ``separability'' of the model. Equivalently, conditioned on $\th^*=\th$, $R_\B(a)$ is a Bernoulli random variable with mean $\phib(a^\T\th)$. In the following, we will use $\C{L}_\b$ to denote the logistic bandits problem instance with parameter $\b$.\par 
\sloppy At stage $t$ the agent plays action $A_t$ and observes reward $R_t = R(A_t)$. Let $\hist_t = \s(A_1, R_1, \dots, A_t, R_t)$ be the $\s$-algebra generated by the past actions and observations (rewards). A (randomized) \textit{policy} $\pi=(\pi_1,\pi_2,\dots)$ is a sequence of functions such that for each $t$, $\pi_t(\hist_{t-1})$ is a probability distribution on the action set. The performance of policy $\pi$ on problem instance $\C{L} =  (\A, \Th, R, \phi, \rho)$ is evaluated by the \textit{Bayesian regret}, defined as
\begin{equation}
    \label{eq:defRegret}
    \RM{BayesRegret}(T;\C{L}, \pi) := \E_{\pi,\rho}\left[\sum_{t=1}^T R^* - R_t \right],
\end{equation}
where $R^*:=R(\th^*)$, the subscripts $\pi,\rho$ denote that $A_t$ is drawn from $\pi_t(\hist_{t-1})$ for $t\geq 1$ and $A_0$ is drawn from the prior $\rho$. 
In this work, we are interested in the \textit{Thompson sampling policy} $\pi^{\TS}$, characterized as
\begin{equation}
    \label{eq:defTS}
    \P\left( \pi_t^{\TS}(\hist_{t-1})\in\cdot\ |\ \hist_{t-1}\right) = \P(A^*\in\cdot\ |\ \hist_{t-1}),
\end{equation}
i.e. the action played in each stage is drawn from the posterior of the optimal action. Since there is a one-to-one mapping between each parameter and the corresponding optimal action, the Thompson sampling policy can be equivalently carried out by sampling from the posterior of the true parameter $\th^*$ at each stage, and acting greedily with respect to the sampled parameter.

%%%%%%%%%%%%%%%%%%%%%%%%%%%%%%%%%%%%%%%%% 
\section{Main Results}%Our first result is the following:
\label{sec:main-results}

We start off the section with a regret bound that only depends on dimension $d$ and the number of time steps $T$, for the particular setting where the action set $\A$ is the same as the parameter set $\Th$. 
\begin{theorem}
    \label{thm:Clean}
    For any $\b>0$, if $\C{L}_\b=(\A,\Th,R_\B,\phib,\rho)$ is such that $\A,\Th\subset\BF{S}_{d-1}$ and $\A=\Th$, then
    \[
        \RM{BayesRegret}(T; \C{L}_\b, \pi^\TS) \leq 40d\sqrt{T \log \left( 3 + \frac{3\sqrt{2T}}{2d}\right)}.
    \]
\end{theorem}
Despite nonlinearity of the link function, Theorem \ref{thm:Clean} matches the $\t{O}(d\sqrt{T})$ bound for linear bandits. 
It is worth noting that the this bound has no dependence on $\b$ or the number of arms, and also matches the $\Omega(d\sqrt{T})$ minimax lower bound for linear bandits in \cite{dani2008stochastic}, ignoring a $\sqrt{\log T}$ factor. This result shows that if there exists an action that aligns perfectly with each potential parameter, the performance of Thompson sampling only depends on the problem dimension $d$, and the dependence is at most linear. 

However, as our next result shows, if the parameters do not align perfectly with their corresponding optimal actions, we have to introduce the fragility dimension to characterize the difficulty of the problem.\par
For our general result, we assume that the following assumption holds.
\begin{assumption}
    \label{assp:main}
    There exists constant $\D\in[0,1]$ such that for every $\th\in\Th$ there is $\a(\th)^\T\th\geq \D$.
\end{assumption}

For a given logistic bandit problem instance $\C{L}_\b=(\A,\Th,R,\phib,\rho)$ that satisfies Assumption \ref{assp:main}, we show that the Bayesian regret of Thompson sampling on $\C{L}_\b$ is closely related to its ``fragility dimension,'' a notion that we introduce below.

\begin{definition}
    \label{def:DOI}
    For any given pair of (possibly infinite) subsets $(\C{X},\C{Y})$ of $\BF{B}_d$, the fragility dimension, denoted by $\eta(\C{X},\C{Y})$, is defined as the largest integer $M$, such that there exists $\{y_1,\dots, y_M\}\subseteq\C{Y}$, with
    \[
        f^*(y_i)^\T y_j < 0,\quad \forall i,j\in[M], i\neq j,
    \]
    where $f^*(y) := \argmax_{x\in\C{X}} x^\T y$. 
    \sloppy The fragility dimension of a problem instance $\C{L}_0=(\A_0,\Th_0,R_0,\phi_0,\rho_0)$ is defined as the fragility dimension of $(\A_0,\Th_0)$, and is denoted by $\eta(\C{L}_0)$.
\end{definition}
\begin{example}
    \label{exp:DOI}
    If the action set and the parameter set of $\C{L}$ are identical subsets of $\BF{S}_{d-1}$, then for each $\th\in\Th$, there is $\a(\th)=\th$. We will show in Appendix \ref{sec:bounds_of_DOI_linearRegime} that in $\BF{S}_{d-1}$ there exists at most $d+1$ vectors with pairwise negative inner products. Therefore, the fragility dimension is bounded by
    \[
        \eta(\C{L}) \leq d+1.
    \]
\end{example}
\begin{remark}
    \label{remark:DOIBound}
    Obviously the fragility dimension cannot exceed the cardinality of the action (parameter) set.  We will show in Appendix \ref{sec:bounds_of_DOI} that we can upper bound the worst-case fragility dimension by the dimensionality $d$ and the constant $\D$ in Assumption \ref{assp:main}. Roughly speaking,
    \begin{itemize}
        \item If $\C{L}$ is such that $\D=1$, then $\eta(\C{L})\leq d+1$ (cf. Example \ref{exp:DOI});
        \item For any fixed $\D\in(0,1)$, if we only consider problem instances such that Assumption \ref{assp:main} holds with constant $\D$, then the worst-case fragility dimension grows exponentially with $d$.
        \item For any $d\geq 3$, we can find a problem instance $\C{L}$ such that Assumption \ref{assp:main} holds with constant $\D=0$, whose fragility dimension is arbitrarily large.
    \end{itemize}
\end{remark}
\begin{remark}
    \label{remark:DOI}
    For given finite action and parameter sets $\A$ and $\Th$, we can think of each parameter as a vertex in a graph $\C{G}$. Two vertices $i$ and $j$ of $\C{G}$ are connected by an edge if and only if
    \[
        \a(\th_i)^\T\th_j < 0\text{ and }\a(\th_j)^\T\th_i < 0.
    \]
    Thus determining the fragility dimension of $(\A,\Th)$ is equivalent to finding the maximum clique in $\C{G}$. This is a widely studied NP-complete problem and there exists a number of efficient heuristics, see \cite{tarjan1977finding}, \cite{tomita2007efficient} and references therein.
\end{remark}
The following general result for the performance of Thompson sampling gives a $\t{O}(\sqrt{d \eta T} / \lambda)$ regret bound.
\begin{theorem}
    \label{thm:main}
    For any $\b>0$, if $\C{L}_\b$ is such that Assumption \ref{assp:main} holds with $\D\in(0,1]$, then 
    \begin{equation}
        \label{eq:main}
        \RM{BayesRegret}(T; \C{L}_\b, \pi^\TS) \leq 20\D^{-1}\sqrt{2d\cdot(\eta(\C{L}_\b)\vee d) \cdot T \log \left( 3 + \frac{3\sqrt{2T}}{2d\D}\right)},
    \end{equation}
    where $a\vee b=\max\{a,b\}$.  It is worth noting that the fragility dimension only depends on the action and parameter sets of the problem instance, hence the right-hand side of \eqref{eq:main} has no dependence on $\beta$.
\end{theorem}
\begin{remark}
    Considering Example \ref{exp:DOI}, and noting that when $\A=\Th$, Assumption \ref{assp:main} holds with $\D=1$, we immediately arrive at Theorem \ref{thm:Clean}.
\end{remark}
\begin{remark}
    Interestingly, the fragility dimension is not monotonic with respect to the inclusion of sets, i.e. there exist sets $\C{X}_1, \C{X}_2, \C{Y}$, such that $\C{X}_1\subset\C{X}_2$ but $\eta(\C{X}_1,\C{Y}) > \eta(\C{X}_2,\C{Y})$. As we show in Appendix \ref{sec:harder}, this fact means that by reducing the size of the action set, we could arrive at a more difficult problem. This is a somewhat surprising result that is worth noting.
\end{remark}\par
We also show that the $\eta$ term in \eqref{eq:main} is critical, since for any fixed $\D<1$, there cannot exist an $\eta$-independent upper bound that is polynomial in $d$ and sublinear in $T$.
\begin{theorem}
    \label{thm:lowerBound}
    For any fixed $\D\in[0,1)$, let $f(\cdot)$ be any real function, $p(\cdot)$ be any polynomial and $\e>0$ be any constant. There exists a logistic bandit problem instance $\C{L}_\b$ and integer $T_0$ such that $\C{L}_\b$ satisfies Assumption \ref{assp:main} with constant $\D$ and 
    \begin{equation}
        \label{eq:thm:lowerBound}
        \RM{BayesRegret}(T_0; \C{L}_\b, \pi) \geq f(\D)p(d)\cdot T_0^{1-\e},
    \end{equation}
    for any policy $\pi$.
\end{theorem}

%%%%%%%%%%%%%%%%%%%%%%%%%%%%%%%%%%%%%%%%%%%%
\section{Main Devices in the proof of Theorem \ref{thm:main}}
\label{sec:3}
% Throughout this section we assume that $\C{L}_\b$ is any problem instance such that Assumption \ref{assp:main} holds with $\D=1$. The section is organized as follows. In Section \ref{sec:3.1}, we introduce the notion of \textit{information ratio}, originally defined in \cite{russo2016information}, and reduce the problem of bounding Bayesian regret to the problem of bounding information ratio; in Section \ref{sec:3.2}, by resorting to the Lipschitzity of the logistic link function, we bound the information ratio for $\b\leq 2$; in Section \ref{sec:3.3}, we prove an important lemma by means of a graph-theoretical analysis; in Section \ref{sec:3.4}, we bound the information ratio for $\b>2$ and thereby conclude the proof. \par
% We fix the stage index $t$ in this section. To simplify notations, we let $Y$ be a random variable with the same distribution as $\th^*$ conditioned on $\hist_{t-1}$. We also define $X=\a(Y)$ and let $\hX$ be an iid copy of $X$. Thus $Y$, $X$ and $\hX$ can be interpreted as aliases for $\th^*$, $A^*$ and $A_t$, respectively. We will omit the ``almost surely'' qualifications whenever ambiguities do not arise.
In this section we discuss the two main devices in the proof of Theorem \ref{thm:main}. In Section \ref{sec:3.1}, we introduce the notion of information ratio, and present the result that relates information ratio with Bayesian regret. In Section \ref{sec:3.2}, we highlight the role of fragility dimension. The full proof of Theorem \ref{thm:main} is given in Appendix \ref{sec:proof_of_thm_main}.\par 
%%%%%%%%%%%%%%%%%%%%%%%%%%%%%%%%%%%%%%%%%%%%%%%
\subsection{Information Ratio}
\label{sec:3.1}
To quantify the exploration-exploitation trade-off at stage $t$, for problem instance $\C{L}$ and policy $\pi$ we define the random variable \textit{information ratio} as the square of one-stage expected regret divided by the amount of information that the agent gains from playing an action and observing the reward, i.e.
\begin{equation}
    \label{eq:defInfoRatio}
    \Gamma_t(\C{L}, \pi) := \frac{\E_{t-1}\left[ R^* - R_t \right]^2}{I_{t-1}(A^*; A_t, R_t)},
\end{equation}
where the subscript $t-1$ in the right-hand side denotes evaluation under base measure $\P(\cdot|\hist_{t-1})$. If the information ratio is small at stage $t$, the agent executing the policy $\pi$ will only incur a large regret if she is about to acquire a large amount of information towards the optimal action. Past results have shown that, as long as the information ratio of Thompson sampling can be uniformly bounded, we immediately obtain a bound on the Bayesian regret of Thompson sampling.\par  
\begin{proposition}
\label{prop:previousLogistic}
\textbf{(Theorem 4, \cite{dong2018information})}
Let $\C{L}_\b=(\A, \Th, R, \phib, \rho)$ be any logistic bandit problem instance with $\inf_{\th\in\Th} |\a(\th)^\T\th| = \d>0$. Further assume that there exists constant $\bar{\Gamma}$ such that
\[
\Gamma_t(\C{L}_\b, \pi^\TS) \leq \bar{\G},\quad \text{a.s. }\forall t=1,2,\dots.
\]
Then we have
\[
    \RM{BayesRegret}(T;\C{L}_\b,\pi^{\RM{TS}})
    \leq \sqrt{8d\bar{\G}\cdot T\log\left( 3 + \frac{6\sqrt{2T}}{d}\cdot\frac{\b e^{\b\d}}{(1+e^{\b\d})^2} \right)}.
\]
\end{proposition}

%%%%%%%%%%%%%%%%%%%%%%%%%%%%%%%%%%%%%%%%%%%%

\subsection{Fragility Dimension}
\label{sec:3.2}
The one-stage expected regret can be written as
\begin{equation}
    \label{eq:oneStageRegret}
    \E_{t-1}[R^*-R_t] = \E_{t-1}[\phib((A^*)^\T\th^*) - \phib(A_t^\T\th^*)]
\end{equation}
It is worth noting that $A^*=\a(\th^*)$ and by the definition of Thompson sampling, $A^*$ and $A_t$ are independent and identically distributed.
Let's first consider the simple case where $\b=\infty$, which motivates our analysis. When $\b=\infty$, we have that $\phi_\b(x)=1$ for all $x \geq 0$ and $\phi_\b(x)=0$ for all $x<0$\footnote{For the sake of simplicity, we will assume that $\phi_\infty(0)=1$, while in fact $\lim_{\b\to\infty}\phib(0)=1/2$. The value of $\phi_\infty(0)$ does not play a role in our analysis.}. By Assumption \ref{assp:main}, we have
\begin{equation} 
    \phib((A^*)^\T\th^*) = \phib\big(\a(\th^*)^\T\th^*\big) = 1.
\end{equation}
There is also
\begin{equation}
    \E_{t-1}[\phib(A_t^\T \th^*)] = \P_{t-1}(A_t^\T \th^*\geq 0).
\end{equation}
Therefore, to upper bound the right-hand side of \eqref{eq:oneStageRegret}, we need to lower bound $\P_{t-1}(A_t^\T \th^*\geq 0)$. The proposition below shows that this term is connected critically with the fragility dimension of $(\A,\Th)$. The proof is given in Appendix \ref{sec:proof_of_prop_strongerTuran}.
\begin{proposition}
    \label{prop:strongerTuran}
    Let $\C{U},\C{V}$ be finite subsets of $\BF{B}_d$. Suppose that there exists bijection $f^*:\C{V}\mapsto\C{U}$ such that
    \[
        f^*(v)^\T v = \max_{u\in\C{U}} u^\T v,\quad\forall v\in\C{V},
    \]
    and $f^*(v)^\T v>0$ for all $v\in \C{V}$. Let $V$ be any random variable supported on $\C{V}$, $U=f^*(V)$ and $\hU$ be an iid copy of $U$. Then
    \begin{equation}
        \label{eq:strongerTuran}
        \P\big(\hU^\T V\geq 0\big) \geq \frac{1}{2\eta(\C{U},\C{V})}.
    \end{equation}
\end{proposition}

%%%%%%%%%%%%%%%%%%%%%%%%%%%%%%%%%%%%%%
\section{Proof Sketch of Theorem \ref{thm:lowerBound}}
\label{sec:4}
Recall that we can obtain regret bounds for linear bandits that are dependent only on the dimensionality of the problem $d$ rather than the number of actions (such as the one in \cite{russo2016information}).  The reason behind such bounds is that when the link function $\phi$ is linear, the difference between the mean rewards of two actions that are close to each other is always small. However, in logistic bandit problems, when parameter $\b$ is large, we could run into cases where two close actions yield diametrically different rewards, as is illustrated in Figure \ref{fig:linear_and_logistic}.\par
\begin{figure}
    \centering
    \includegraphics[height=150pt]{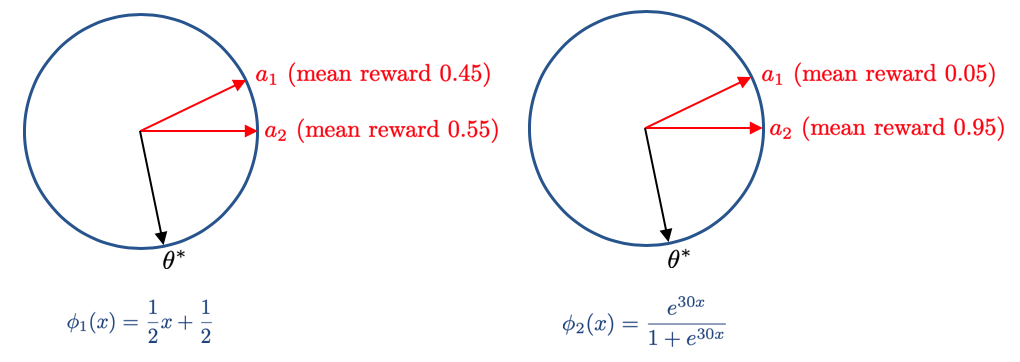}
    \caption{The difference between linear and logistic bandits. The actions $a_1$ and $a_2$ are ``similar'' to each other in that their embeddings in the Euclidean space are close. Under the linear link function $\phi_1$, the mean rewards of $a_1$ and $a_2$ are also similar. However, under the logistic link function $\phi_2$, the performances of the two actions are diametrical.} 
    \label{fig:linear_and_logistic}
\end{figure}
Specifically, suppose that our action and parameter sets are such that
\begin{equation}
    \label{eq:worst1}
    \a(\th)^\T\th \geq 0,\quad \forall \th\in\Th,
\end{equation}
and
\begin{equation}
    \label{eq:worst2}
    a^\T\th < 0,\quad \forall a\in\A, \th\in\Th, a\neq\a(\th),
\end{equation}
that is, $\eta(\A,\Th)=|\A|=|\Th|$. 
Then, when $\b$ is large, conditioned on each parameter being the true parameter, there is exactly one action with mean reward close to 1, while the mean rewards of all other actions are close to 0. The following proposition shows that in this problem the optimal action is inherently hard to learn, in the sense that the regret of any algorithm grows linearly in the first $|\A|/2-1$ stages. The proof can be found in Appendix \ref{sec:proof_of_prop_NoSublinear}.
\begin{proposition}
    \label{prop:NoSublinear}
    Let $\C{L}=(\A, \Th, R, \phi,\rho)$ be a generalized linear bandit problem such that $|\A|=N<\infty$, $R$ is binary and $\rho$ is the uniform distribution over $\A$. Suppose that for each $a\in\A$,
    \[
        \E[R(a)|A^*=a] \geq 1-\frac{1}{N},
    \]
    and 
    \[
        \max_{a'\neq a} \E[R(a')|A^*= a] \leq \frac{1}{N}.
    \]
    Then for any policy $\pi$,
    \begin{equation}
       \RM{BayesRegret}(t;\C{L}, \pi) \geq \frac{t}{4},\quad \forall t\leq \frac{N}{2}-1.
    \end{equation}
\end{proposition}

We can also show that (as in Appendix \ref{sec:bounds_of_DOI}), for any fixed $\D\in(0,1)$, there exists $\g>1$, such that for any $d\geq 2$ we can find a pair of action and parameter sets $(\A_d,\Th_d)$ with $\A_d,\Th_d\in\BB{R}^d$, $|\A_d|=|\Th_d|\geq\g^d$ that satisfies \eqref{eq:worst1}, \eqref{eq:worst2} and Assumption \ref{assp:main} with constant $\D$. For any real function $f(\cdot)$, polynomial $p(\cdot)$ and constant $\e\in(0,1)$, choose $d$ large enough such that $\g^{\e d}>16f(\D)p(d)$ and $\b_d$ large enough such that 
\[
    \phi_{\b_d}(\D) \geq 1-\frac{1}{|\A_d|}
\]
and
\[
    \phi_{\b_d}\left(\max_{a\in\A_d,\th\in\Th_d, a^\T\th<0}a^\T\th \right) \leq \frac{1}{|\A_d|}.
\]
Consider the problem $\C{L}=(\A_d,\Th_d,R_\B,\phi_{\b_d},\mathrm{Unif}(\A))$ at stage $T_0=\g^d/4$, from Proposition \ref{prop:NoSublinear} we have
\begin{equation}
    \RM{BayesRegret}(T_0;\C{L}, \pi)
    \geq \frac{T_0}{4}
    = \frac{\g^d}{16}
    = \frac{1}{4}\cdot\frac{\g^{\e d}}{4^\e}\cdot\left(\frac{\g^d}{4}\right)^{1-\e}
    > f(\D) p(d) T_0^{1-\e},
\end{equation}
for any policy $\pi$.

% Acknowledgments---Will not appear in anonymized version
\acks{Toyota Research Institute (TRI) provided funds to assist the authors (Tengyu Ma) with their research but this article solely reflects the opinions and conclusions of its authors and not TRI or any other Toyota entity. Shi Dong is supported by the Herb and Jane Dwight Stanford Graduate Fellowship.}

\bibliography{bibliography_colt}

\begin{thebibliography}{17}
\providecommand{\natexlab}[1]{#1}
\providecommand{\url}[1]{\texttt{#1}}
\expandafter\ifx\csname urlstyle\endcsname\relax
  \providecommand{\doi}[1]{doi: #1}\else
  \providecommand{\doi}{doi: \begingroup \urlstyle{rm}\Url}\fi

\bibitem[Abeille and Lazaric(2017)]{abeille2017linear}
Marc Abeille and Alessandro Lazaric.
\newblock Linear thompson sampling revisited.
\newblock \emph{Electronic Journal of Statistics}, 11\penalty0 (2):\penalty0
  5165--5197, 2017.

\bibitem[B{\"o}r{\"o}czky~Jr et~al.(2004)B{\"o}r{\"o}czky~Jr, B{\"o}r{\"o}czky,
  et~al.]{boroczky2004finite}
K{\'a}roly B{\"o}r{\"o}czky~Jr, K~B{\"o}r{\"o}czky, et~al.
\newblock \emph{Finite packing and covering}, volume 154.
\newblock Cambridge University Press, 2004.

\bibitem[Bubeck and Eldan(2016)]{bubeck2016multi}
S{\'e}bastien Bubeck and Ronen Eldan.
\newblock Multi-scale exploration of convex functions and bandit convex
  optimization.
\newblock In \emph{Conference on Learning Theory}, pages 583--589, 2016.

\bibitem[Dani et~al.(2008)Dani, Hayes, and Kakade]{dani2008stochastic}
Varsha Dani, Thomas~P Hayes, and Sham~M Kakade.
\newblock Stochastic linear optimization under bandit feedback.
\newblock In \emph{21st Annual Conference on Learning Theory}, pages 355--366,
  2008.

\bibitem[Dong and Van~Roy(2018)]{dong2018information}
Shi Dong and Benjamin Van~Roy.
\newblock An information-theoretic analysis for {T}hompson sampling with many
  actions.
\newblock In \emph{Advances in Neural Information Processing Systems}, 2018.

\bibitem[Filippi et~al.(2010)Filippi, Cappe, Garivier, and
  Szepesv{\'a}ri]{filippi2010parametric}
Sarah Filippi, Olivier Cappe, Aur{\'e}lien Garivier, and Csaba Szepesv{\'a}ri.
\newblock Parametric bandits: The generalized linear case.
\newblock In \emph{Advances in Neural Information Processing Systems}, pages
  586--594, 2010.

\bibitem[Li et~al.(2017)Li, Lu, and Zhou]{li2017provable}
Lihong Li, Yu~Lu, and Dengyong Zhou.
\newblock Provably optimal algorithms for generalized linear contextual
  bandits.
\newblock In \emph{International Conference on Machine Learning}, pages
  2071--2080, 2017.

\bibitem[Liu et~al.(2018)Liu, Buccapatnam, and Shroff]{liu2018information}
Fang Liu, Swapna Buccapatnam, and Ness Shroff.
\newblock Information directed sampling for stochastic bandits with graph
  feedback.
\newblock In \emph{Thirty-Second AAAI Conference on Artificial Intelligence},
  2018.

\bibitem[Russo and Van~Roy(2013)]{russo2013eluder}
Daniel Russo and Benjamin Van~Roy.
\newblock Eluder dimension and the sample complexity of optimistic exploration.
\newblock In \emph{Advances in Neural Information Processing Systems}, pages
  2256--2264, 2013.

\bibitem[Russo and Van~Roy(2014{\natexlab{a}})]{russo2014learning}
Daniel Russo and Benjamin Van~Roy.
\newblock Learning to optimize via information-directed sampling.
\newblock In \emph{Advances in Neural Information Processing Systems}, pages
  1583--1591, 2014{\natexlab{a}}.

\bibitem[Russo and Van~Roy(2014{\natexlab{b}})]{russo2014posterior}
Daniel Russo and Benjamin Van~Roy.
\newblock Learning to optimize via posterior sampling.
\newblock \emph{Mathematics of Operations Research}, 39\penalty0 (4):\penalty0
  1221--1243, 2014{\natexlab{b}}.

\bibitem[Russo and Van~Roy(2016)]{russo2016information}
Daniel Russo and Benjamin Van~Roy.
\newblock An information-theoretic analysis of {T}hompson sampling.
\newblock \emph{The Journal of Machine Learning Research}, 17\penalty0
  (1):\penalty0 2442--2471, 2016.

\bibitem[Russo and Van~Roy(2018)]{russo2018satisficing}
Daniel Russo and Benjamin Van~Roy.
\newblock Satisficing in time-sensitive bandit learning.
\newblock \emph{arXiv preprint arXiv:1803.02855}, 2018.

\bibitem[Tarjan and Trojanowski(1977)]{tarjan1977finding}
Robert~Endre Tarjan and Anthony~E Trojanowski.
\newblock Finding a maximum independent set.
\newblock \emph{SIAM Journal on Computing}, 6\penalty0 (3):\penalty0 537--546,
  1977.

\bibitem[Thompson(1933)]{thompson1933likelihood}
William~R Thompson.
\newblock On the likelihood that one unknown probability exceeds another in
  view of the evidence of two samples.
\newblock \emph{Biometrika}, 25\penalty0 (3/4):\penalty0 285--294, 1933.

\bibitem[Tomita and Kameda(2007)]{tomita2007efficient}
Etsuji Tomita and Toshikatsu Kameda.
\newblock An efficient branch-and-bound algorithm for finding a maximum clique
  with computational experiments.
\newblock \emph{Journal of Global optimization}, 37\penalty0 (1):\penalty0
  95--111, 2007.

\bibitem[Tur\'an(1941)]{turan1941}
Paul Tur\'an.
\newblock On an extremal problem in graph theory.
\newblock \emph{Matematikai \'es Fizikai Lapok (in Hungarian)}, 48:\penalty0
  436--452, 1941.

\end{thebibliography}

\appendix
\section{Proof of Proposition \ref{prop:strongerTuran}}
\label{sec:proof_of_prop_strongerTuran}
We present a graph-theoretical proof of Proposition \ref{prop:strongerTuran}. For simplicity, let $\eta=\eta(\C{U}, \C{V})$.  Let $\C{U}$ and $\C{V}$ be enumerated as $\C{U}=\{u_1, \dots, u_n\}$ and $\C{V}=\{v_1, \dots, v_n\}$. Without loss of generality, we assume that $f^*(v_i) = u_i$ for $i\in[n]$. We construct an undirected graph $\C{G} = (\C{K}, \C{E})$, where $\C{K}=\{1,\dots, n\}$, and for any pair $1\leq i<j\leq n$, $i$ and $j$ are connected by an edge $(i,j)\in\C{E}$ if and only if 
\[
    f^*(v_i)^\T v_j < 0 \text{ and } f^*(v_j)^\T v_i < 0.
\]
From Definition \ref{def:DOI}, there exists no $(\eta+1)$-clique in $\C{G}$.\par
Let $p$ be any probability measure on $\C{V}$. We use $p_i$ to denote the probability mass associated with $v_i$. Thus $p_i\geq 0$ and $\sum_{i=1}^n p_i = 1$. For fixed $\C{V}$, let $J(p) = \P_p\big( \hU^\T V < 0 \big)$, where the subscript $p$ indicates that the distribution of $V$ is $p$. We have that
\begin{eqnarray}
    J(p)
    &=& \P_p\big( \hU^\T V < 0 \big)\nn\\
    &=& \sum_{i=1}^n \sum_{j=1}^n \P_p(\hU=u_i) \P_p(V=v_j) \ind(u_i^\T v_j < 0)\nn\\
    &\overset{\add\thecount\label{1}}{=}& \sum_{i,j=1}^n p_i p_j \ind\big(f(v_i)^\T v_j < 0\big)\nn\\
    &\overset{\add\thecount\label{2}}{\leq}& \sum_{(i,j)\in\C{E}}p_ip_j + \frac{1}{2}\sum_{(i,j)\notin\C{E}}p_ip_j\nn\\
    &=& \frac{1}{2} + \frac{1}{2}\sum_{(i,j)\in\C{E}}p_ip_j,
\end{eqnarray}
where $\ref{1}$ comes from that
\[
    \P_p(\hU=u_i) = \P_p(U=u_i) = \P_p(V=v_i),
\]
and $\ref{2}$ is because for each $(i,j)\notin\C{E}$, at most one of $f(v_i)^\T v_j$ and $f(v_j)^\T v_i$ can be negative. Note that here $(i,i)\notin\C{E}$ for all $i\in[n]$. \par
Let $M(p):=\sum_{(i,j)\in\C{E}} p_ip_j$. 
We first argue that there exists probability measure $p^*$, such that 
\[
M(p^*) = \max_p M(p),
\]
and for any $(i,j)\notin\C{E}$, $i\neq j$, either $p^*_i=0$ or $p^*_j=0$. In fact, let $p$ and $(i,j)\notin\C{E}$ be arbitrary. Without loss of generality, assume that
\[
\sum_{k:(i,k)\in\C{E}} p_k \geq \sum_{k:(j,k)\in\C{E}}p_k.
\]
We define a new measure $p'$ as follows: $p'_i = p_i + p_j$, $p'_j = 0$ and $p'_\ell = p_\ell$ for $\ell\neq i,j$. Then
\begin{eqnarray}
    M(p')
    &=& \sum_{(\ell,k)\in\C{E}} p'_\ell p'_k\nn\\
    &=& \sum_{k:(i,k)\in\C{E}} p'_ip'_k + \sum_{k:(i,k)\in\C{E}} p'_jp'_k + \sum_{h, \ell\neq i,j: (h,\ell)\in\C{E}} p'_hp'_\ell\nn\\
    &=& \sum_{k:(i,k)\in\C{E}} (p_i+p_j)p_k  + \sum_{h, \ell\neq i,j: (h,\ell)\in\C{E}} p_hp_\ell\nn\\
    &\geq& \sum_{k:(i,k)\in\C{E}} p_ip_k + \sum_{k:(j,k)\in\C{E}} p_jp_k  + \sum_{h, \ell\neq i,j: (h,\ell)\in\C{E}} p_hp_\ell\nn\\
    &=& \sum_{(\ell,k)\in\C{E}} p_\ell p_k\nn\\
    &=& M(p).
\end{eqnarray}
Therefore, by moving all the probability mass from $j$ to $i$, the value $M$ does not decrease. Thus we can always find a probability measure $p^*$ which attains the maximum of $M$, and at the same time satisfies $p^*_ip^*_j=0$ whenever $(i,j)\notin\C{E}$ and $i\neq j$.\par
Next we show that there can be at most $\eta$ non-zero elements among $\{p^*_1,\dots,p^*_n\}$. In fact, since there exists no $(\eta+1)$-clique in $\C{G}$, for any subset $\{i_1, \dots, i_{\eta+1}\}$ of $\C{V}$ there must exist $(i_s, i_t)\notin\C{E}$ and $i_s\neq i_t$. This leads to $p^*_{i_s}p^*_{i_t}=0$. Hence $p^*$ must be supported on at most $\eta$ elements of $\C{X}$.\par 
Without loss of generality, let $p^*_1,\dots,p^*_{\eta}\geq 0$ and $p^*_{\eta+1},\dots, p^*_n=0$. Then
\begin{eqnarray}
    \max_p J(p)
    &\leq& \max_p \left(\frac{1}{2} + \frac{1}{2}M(p)\right)\nn\\
    &=&  \frac{1}{2} + \frac{1}{2}M(p^*)\nn\\
    &=& \frac{1}{2} + \frac{1}{2}\left(1 - \sum_{(i,j)\notin\C{E}} p^*_ip^*_j \right) \nn\\
    &\leq& 1 - \frac{1}{2}\sum_{k=1}^{\eta} (p^*_k)^2\nn\\
    &=& 1 - \frac{1}{2\eta},
\end{eqnarray}
where the last inequality comes from $\sum_{k=1}^{\eta} (p^*_k)^2\geq \frac{1}{\eta}\left(\sum_{k=1}^{\eta}p^*_k\right)^2= \frac{1}{\eta}$. Hence
\[
    \P_p\big( \hU^\T V \geq 0 \big) = 1 - J(p) \geq \frac{1}{2\eta},\quad \forall p,
\]
which is the result we desire.

\begin{remark}
If $\C{U}=\C{V}$ and $f^*$ is the identity function, we can get rid of the additional $1/2$ factor and show that 
\[
    \P(\hU^\T V\geq 0) \geq \frac{1}{\eta}.
\]
In fact, if $V$ is uniformly distributed on $\C{V}$, we can recover the prestigious Tur\'an's theorem in graph theory:
\begin{theorem}
\textbf{(\cite{turan1941})} If a graph with $n$ vertices does not contain any $(k+1)$-clique, then its number of edges cannot exceed $\left(1-\frac{1}{k}\right)\cdot \frac{n^2}{2}$.
\end{theorem}
\end{remark}
By restricting the random vector $V$ to a subset of $\BB{R}^d$, we have the following corollary.
\begin{corollary}
    \label{cor:strongerTuran}
    Let $\C{U},\C{V}$ be finite subsets of $\BF{B}_d$. Suppose that there exists bijection $f^*:\C{V}\mapsto\C{U}$ such that
    \[
        f^*(v)^\T v = \max_{u\in\C{U}} u^\T v,\quad\forall v\in\C{V},
    \]
    and $f^*(v)^\T v>0$ for all $v\in \C{V}$. Let $V$ be any random variable supported on $\C{V}$, $U=f^*(V)$ and $\hU$ be an iid copy of $U$. Then for any $\C{S}\subseteq\C{V}$,
    \begin{equation}
        \label{eq:strongerTuran}
        \E\bigg[\ind\big(\hU^\T V\geq 0\big) \ind\big(\hU\in f^*(\C{S})\big)\ind\big(V\in \C{S}\big)\bigg] \geq \frac{1}{2\eta(\C{U},\C{V})} \E\bigg[\ind\big(\hU\in f^*(\C{S})\big)\ind\big(V\in \C{S}\big)\bigg],
    \end{equation}
    where $f^*(\C{S}):= \{ f^*(v): v\in\C{S}\}$.
\end{corollary}
\newpage

\section{Proof of Theorem \ref{thm:main}}
\label{sec:proof_of_thm_main}
Considering Proposition \ref{prop:previousLogistic}, and the fact that
\[
    \frac{\b e^{\b\D}}{(1+e^{\b\D})^2} \leq \frac{1}{4\D},
\]
we only have to show
\begin{equation}
    \label{eq:infoRatioBound}
    \G_t(\C{L}_\b, \pi^\TS)\leq 100\D^{-2}(\eta(\C{L}_\b)\vee d),\quad \text{a.s.}, \forall t.
\end{equation}
We will present two separate proofs of \eqref{eq:infoRatioBound} for $\b\leq 2$ and $\b>2$, respectively. For $\b\leq 2$, we resort to the previous Lipschizity analysis; for $\b>2$, we adopt a new line of analysis that is connected to our definition of fragility dimension.\par
We fix the stage index $t$ in this section. To simplify notations, we let $Y$ be a random variable with the same distribution as $\th^*$ conditioned on $\hist_{t-1}$. We also define $X=\a(Y)$ and let $\hX$ be an iid copy of $X$, $\hY$ an iid copy of $Y$. Thus $X$, $Y$, $\hX$ and $\hY$ can be interpreted as aliases for $A^*$, $\th^*$, $A_t$ and $\th_t$, respectively. As a shorthand we use $\eta$ in place of $\eta(\C{L}_\b)$. We will omit the ``almost surely'' qualifications whenever ambiguities do not arise.\par

Before moving on, we introduce a result adapted from \cite{russo2016information}, which gives a primitive bound of information ratio.
\begin{proposition}
    \label{prop:primitiveBound}
    For any generalized linear bandit problem $\C{L} = (\A, \Th, R, \phi, \rho)$,
    \begin{equation}
    \label{eq:primitiveBound}
        \Gamma_t(\C{L}, \pi^\TS) \leq \frac{\E\big[\phi(X^\T Y) - \phi(X^\T \hY)\big]^2}{2\cdot\E\big[\Var[\phi(X^\T \hY)|X]\big]}.
    \end{equation}
\end{proposition}
\begin{proof}
First notice that, since $\hX$ is independent of $Y$ and $\hY$ is independent of $X$, we have $\E[\phi(X^\T \hY)] = \E[\phi(\hX^\T Y)]$. Therefore \eqref{eq:primitiveBound} is equivalent to
\begin{equation}
    \label{eq:primitiveBound_fixed}
    \Gamma_t(\C{L}, \pi^\TS) \leq \frac{\E\big[\phi(X^\T Y) - \phi(\hX^\T Y)\big]^2}{2\cdot\E\big[\Var[\phi(X^\T \hY)|X]\big]}.
\end{equation}

Comparing \eqref{eq:defInfoRatio} and \eqref{eq:primitiveBound_fixed} and , we only have to show
\begin{equation}{}
    \label{eq:proof-primitiveBound}
    I(X;\hX, R(\hX)) \geq 2\cdot\E\big[\Var[\phi(X^\T \hY)|X]\big].
\end{equation}
In fact, we have that
\begin{eqnarray}
I(X;\hX, R(\hX))
&\overset{\add\thecount\label{15}}{=}& I(Y;\hX, R(\hX))\nn\\
&=& I(Y;\hX) + I(Y;R(\hX) | \hX)\nn\\
&\overset{\add\thecount\label{16}}{=}& I(Y; R(\hX)|\hX)\nn\\
&=& \sum_{x\in\A} I(Y; R(x))\P(\hX=x)\nn\\
&\overset{\add\thecount\label{17}}{=}& \sum_{y\in\Th} I(Y; R(y))\P(Y=y)\nn\\
&=&\sum_{y, y'\in\Th} D_\mathrm{KL}\Big( P\big(R(y')\big) \big\| P\big(R(y')|y=y\big)\Big)\cdot\P(Y=y)\P(Y=y')\nn\\
&\overset{\add\thecount\label{18}}{\geq}& 2 \sum_{y, y'\in\Th}\Big( \E\big[R(y')\big] - \E\big[R(y')|y=y\big]\Big)^2\cdot\P(Y=y)\P(Y=y')\nn\\
&=& 2\sum_{y\in\Th}\P(Y=y)\left\{ \sum_{y'\in\Th}\P(Y=y')\Big( \E\big[R(y')\big] - \E\big[R(y')|Y=y\big]\Big)^2\right\}\nn\\
&\overset{\add\thecount\label{19}}{=}& 2\cdot\E\big[\Var[\phi(\hX^\T Y)|\hX]\big]\nn\\
&=& 2\cdot\E\big[\Var[\phi(X^\T \hY)|X]\big],
\end{eqnarray}
where we use $R(y)$ to denote $R(\a(y))$ for $y\in\Th$. In $\ref{15}$ and $\ref{17}$ we use the fact that $\a$ is a bijection. That $\ref{16}$ holds is because of the independence between $Y$ and $\hX$. In $\ref{18}$ we apply the Pinsker's inequality upon noticing that $R\in\{0,1\}$. The final step $\ref{19}$ follows from the fact that
\[
    \E\big[R(y')| Y=y\big] = \phi(\a(y')^\T y),
\]
and that
\[
    \E\big[R(y')\big] = \E\big[\phi(\a(y')^\T Y)].
\]
Thus we have \eqref{eq:proof-primitiveBound}.
\end{proof}

\subsection{Proof of \eqref{eq:infoRatioBound} for Small $\b$}
We first point out to a useful lemma.
\begin{lemma} 
    \label{lemma:marginals}
    Let $U,V$ be random vectors in $\BB{R}^d$, and let $R,S$ be independent random variables with distributions equal to the marginals of $U,V$, respectively. Then
    \[
        \E\big[ |U^\T V| \big] ^2 \leq d\cdot \E\big[ (R^\T S)^2 \big].
    \]
\end{lemma}
\begin{proof}
Let $\S =\E[VV^\T]$, then
\begin{eqnarray}
    \label{eq:lemma:marginals-1}
    \E\big[ |U^\T V| \big]
    &\overset{\add\thecount\label{3}}{\leq}& \E \big[ \|\S^{1/2} U\|_2\cdot \|\S^{-1/2}V\|_2 \big]\nn\\
     &\overset{\add\thecount\label{4}}{\leq}& \E \big[ \|\S^{1/2} U\|_2^2\big]^{1/2}\cdot \E\big[ \|\S^{-1/2}V\|_2^2 \big]^{1/2}\nn\\
     &=& \E\big[ R^\T \E[SS^\T] R \big]^{1/2}\cdot \E\big[ V^\T \E[VV^\T]^{-1} V \big]^{1/2}\nn\\
     &\overset{\add\thecount\label{5}}{=}& \left(\E\big[ (R^\T S)^2 \big] \cdot d\right)^{1/2},
\end{eqnarray}
where $\ref{3}$ and $\ref{4}$ result from Cauchy-Schwarz inequality and $\ref{5}$ comes from the fact that for any random vector $W$ and non-random matrix $A$, there is $\E[W^\T AW] = \mathrm{Tr}(A\mathrm{Cov}(W)) + \E[W]^\T A\E[W]$. Thus we arrive at our desired result.
\end{proof}

\begin{proposition}
    \label{prop:lipschitz}
    Let $\C{L}=(\A, \Th, R, \phi, \rho)$ be any generalized linear bandit problem instance where $\phi$ is such that there exist constants $0<L_1\leq L_2$ with
    \[
        L_1 \leq |\phi'(x)| \leq L_2,\quad \forall x\in[-1,1].
    \]
    Then we have
    \[
    \G_t(\C{L}, \pi^\TS) \leq d \cdot \frac{L_2^2}{L_1^2}.
    \]
    Specifically, for the logistic bandit problem $\C{L}_\b$, there is
    \[
    \G_t(\C{L}_\b, \pi^\TS) \leq d \cdot \left(\frac{(1+e^\b)^2}{e^\b}\right)^2.
    \]
\end{proposition}
\begin{proof}
From Proposition \ref{prop:primitiveBound}, we have 
\[
    \G_t(\C{L}, \pi^\TS) \leq \frac{\E\big[\phi(X^\T Y) - \phi(X^\T \hY)\big]^2}{2\E\big[\Var[\phi(X^\T \hY)|X]\big]}.
\]
Let $\tY$ be another iid copy of $Y$, there is
\begin{eqnarray}
    \label{eq:infoRatioDown}
    \E\big[\Var[\phi(X^\T \hY)|X]\big]
    &=& \frac{1}{2}\cdot\E\Big[\E\big[\big(\phi(X^\T \hY) - \phi(X^\T \tY)\big)^2|X\big]\Big] \nn\\
    &\geq&  \frac{1}{2}\cdot\E\Big[\E\big[\big(X^\T \hY - X^\T \tY\big)^2\cdot L_1^2|X\big]\Big] \nn\\
    &=& \frac{L_1^2}{2}\cdot\E\Big[\big(X^\T \hY - X^\T \tY\big)^2\Big].
\end{eqnarray}
On the other hand, there is also
\begin{eqnarray}
    \label{eq:infoRatioUp}
    \E\big[ \phi(X^\T Y) - \phi(X^\T \hY) \big]
    &\leq& L_2\cdot \E\big[ |X^\T Y - X^\T \hY| \big] \nn\\
    &=& L_2\cdot \E\big[ |X^\T (Y-\hY)|\big] \nn\\
    &\overset{\add\thecount\label{6}}{\leq}& L_2\cdot \sqrt{d\cdot \E\big[ (X^\T(\hY - \tY))^2\big]},
\end{eqnarray}
where $\ref{6}$ follows from Lemma \ref{lemma:marginals}. Comparing \eqref{eq:infoRatioDown} and \eqref{eq:infoRatioUp}, we arrive at
\[
    \G_t(\C{L}, \pi^\TS) \leq d \cdot \frac{L_2^2}{L_1^2},
\]
which is the desired result.
Plugging in $\C{L}_\b$ into Proposition \ref{prop:lipschitz} and notice that 
\[
    \frac{\b e^\b}{(1+e^\b)^2}\leq \phi'_\b(x) \leq \b,\quad\forall x\in[-1,1],
\]
we shall arrive at 
\[
    \G_t(\C{L}_\b, \pi^\TS) \leq d \cdot \left(\frac{(1+e^\b)^2}{e^\b}\right)^2.
\]
\end{proof}
From Proposition \ref{prop:lipschitz}, for $\b\leq 2$, there is
\begin{equation}
    \label{eq:cor:smallBeta}
    \G_t(\C{L}_\b, \pi^\TS) \leq d\cdot\left(\frac{(1+e^2)^2}{e^2}\right)^2 < 100d.
\end{equation}

\subsection{Proof of \eqref{eq:infoRatioBound} for Large $\b$}
In this section we show \eqref{eq:infoRatioBound} for $\b>2$. Throughout we assume that Assumption \ref{assp:main} holds with constant $\D\in(0,1)$. For any $x\in\A$, let $\s(x)=x^\T \a^{-1}(x)$. For $\zeta\in\BB{R}$, We further define
\[
    \gbl(\zeta) := \phib(\l) - \phib(\l-\zeta).
\]
and let $\zbl=\argmax_{\zeta\in[0, 1+\l]}\gbl(\zeta)/\zeta$, $\wbl=(\l+\zbl)/2$ and $\nb(x) = \E[\phib(X^\T Y) -\phib(X^\T \hY)|X=x] $. Under the above notations, \eqref{eq:primitiveBound} can be written as
\begin{equation}
    \label{eq:primitiveBound2}
    \Gamma_t(\C{L}, \pi^\TS) \leq \frac{\E\big[\gbY(\s(X) - X^\T \hY)\big]^2}{2\cdot\E\big[ \big( \gbY(\s(X) - X^\T \hY) - \nb(X) \big)^2 \big]}.
\end{equation}
We also partition the action set $\A$ into two subsets:
\[
    \C{D} := \left\{ x\in\A: \nb(x) \leq \gbl(\wbl)\right\}
\]
and $\bar{\C{D}} = \Th\backslash\C{D}$. Suppose that we can find constants $C_1, C_2$, such that
\[
\E\big[\gbY(\s(X) - X^\T \hY)\ind(X\in\C{D})\big]^2 \leq C_1\cdot \E\Big[ \big( \gbY(\s(X) - X^\T \hY) - \nb(X) \big)^2 \ind(X\in\C{D})\Big]
\]
and
\[
\E\big[\gbY(\s(X) - X^\T \hY)\ind(X\in\bar{\C{D}})\big]^2 \leq C_2\cdot \E\Big[ \big( \gbY(\s(X) - X^\T \hY) - \nb(X) \big)^2 \ind(X\in\bar{\C{D}})\Big]
\]
Then, from Cauchy-Schwarz inequality we have
\begin{eqnarray}
    && \E\big[\gbY(\s(X) - X^\T \hY)\big]^2\nn\\
    &=& \left(\E\big[\gbY(\s(X) - X^\T \hY)\ind(X\in\C{D})\big] + \E\big[\gbY(\s(X) - X^\T \hY)\ind(X\in\bar{\C{D}})\big]\right)^2\nn\\
    &\leq& 2 \left\{ \left(\E\big[\gbY(\s(X) - X^\T \hY)\ind(X\in\C{D})\big]\right)^2 + \left(\E\big[\gbY(\s(X) - X^\T \hY)\ind(X\in\bar{\C{D}})\big]\right)^2\right\}\nn\\
    &\leq& 2\max\{C_1,C_2\}\E\Big[ \big( \gbY(\s(X) - X^\T \hY) - \nb(X) \big)^2 \Big],
\end{eqnarray}
Thus we can bound the right-hand side of \eqref{eq:primitiveBound2} by
\begin{equation}
    \frac{\E\big[\gbY(\s(X) - X^\T \hY)\big]^2}{2\cdot\E\big[ \big( \gbY(\s(X) - X^\T \hY) - \nb(X) \big)^2 \big]}
    \leq
    \max\{C_1,C_2\}.
\end{equation}
To determine $C_1$, we first introduce a lemma.
\begin{lemma}
    \label{lemma:gbb}
    Let $f:\BB{R}_+\mapsto\BB{R}_+$ be such that $f(0)=0$ and $f(\zeta)/\zeta$ is non-decreasing over $\zeta\geq0$ (f(0)/0 is interpreted as the limit of $\zeta\downarrow 0$). Then for any non-negative random variable $U$, there is
    \begin{equation}
        \label{eq:lemma:variance}
        \frac{\E[f(U)]^2}{\E[U]^2} \leq \frac{\Var[f(U)]}{\Var[U]}.
    \end{equation}
\end{lemma}
\begin{proof}
Let $g(\zeta)=f(\zeta)/\zeta$ with $g(0) = \lim_{\zeta\downarrow 0} f(\zeta)/\zeta$. By our assumption, $g(\zeta)$ is also non-negative and non-decreasing.  Let $V$ be an iid copy of $U$, we have that
\begin{eqnarray}
\E[g(U)^2U]\cdot\E[U^2] - \E[g(U)^2U^2]\cdot\E[U]
&=& \frac{1}{2}\Big(\E[g(U)^2UV^2+g(V)^2U^2V]\nn\\
    &&- \E[g(U)^2U^2V + g(V)^2V^2U]\Big)\nn\\
&=& \frac{1}{2}\E\Big[UV(V-U)\big(g(U)^2-g(V)^2\big)\Big]\nn\\
&\leq& 0,
\end{eqnarray}
where the final inequality results from the monotonicity of $g$. Therefore we have shown
\begin{equation}
    \label{eq:lemma:variance-intermediate}
    \frac{\E[g(U)^2U]}{\E[U]} \leq \frac{\E[g(U)^2U^2]}{\E[U^2]} = \frac{\E[f(U)^2]}{\E[U^2]}.
\end{equation}
Thus there is
\begin{eqnarray}
    \label{eq:lemma:variance-final}
    \frac{\E[f(U)]^2}{\E[U]^2}
    &=& \frac{\E[g(U)U]^2}{\E[U]^2}\nn\\
    &\overset{\add\thecount\label{7}}{\leq}& \frac{\E[g(U)^2U]\E[U]}{\E[U]^2}\nn\\
    &=& \frac{\E[g(U)^2U]}{\E[U]}\nn\\
    &\overset{\add\thecount\label{8}}{\leq}& \frac{\E[f(U)^2]}{\E[U^2]},
\end{eqnarray}
where $\ref{7}$ comes from Cauchy-Schwarz inequality and $\ref{8}$ is the consequence of \eqref{eq:lemma:variance-intermediate}. 
Finally, notice that $\Var[f(U)] = \E[f(U)^2] - \E[f(U)]^2$ and $\Var[U] = \E[U^2] - \E[U]^2$, we have
\begin{eqnarray}
&&\frac{\E[f(U)]^2}{\E[U]^2} \leq \frac{\Var[f(U)]}{\Var[U]}\nn\\
&\Leftrightarrow& \E[f(U)]^2\Var[U] \leq \E[U]^2\Var[f(U)]\nn\\
&\Leftrightarrow& \E[f(U)]^2\big( \E[U^2] - \E[U]^2 \big) \leq \E[U]^2 \big( \E[f(U)^2] - \E[f(U)]^2 \big)\nn\\
&\Leftrightarrow& \E[f(U)]^2\E[U^2] \leq \E[U]^2 \E[f(U)^2],
\end{eqnarray}
where the final inequality is implied by \eqref{eq:lemma:variance-final}. Hence the proof is complete.
\end{proof}

We define function $\gbbl(\zeta)$ by
\[
\gbbl(\zeta) = 
\begin{cases}
\gbl(\zeta) &\zeta\in[0,\zbl]\\
\frac{\gbl(\zbl)}{\zbl}\cdot \zeta & \zeta\in[\zbl,1+\l]
\end{cases},
\]
as is shown in Figure \ref{fig:functions2}. 
\begin{figure}
    \centering
    \includegraphics[width=200pt]{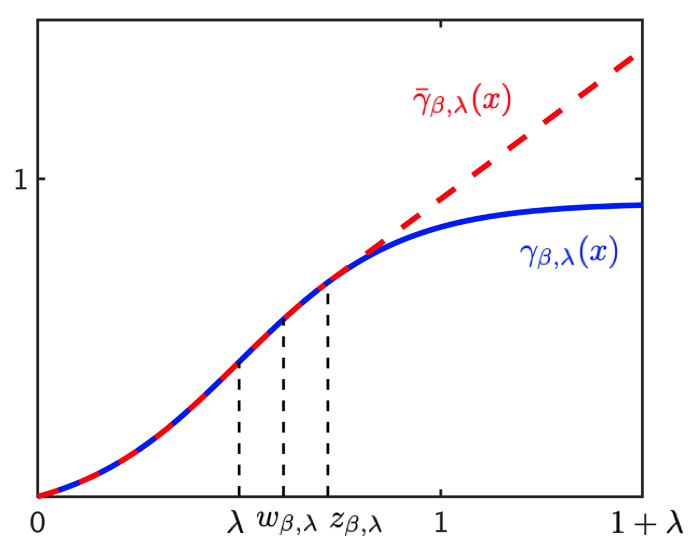}
    \caption{Functions $\gbl$ and $\gbbl$.}
    \label{fig:functions2}
\end{figure}
We thus have
\begin{eqnarray}
    \label{eq:proposition3-1}
    &&\E\Big[ \big( \gbY(\s(X) - X^\T \hY) - \nb(X) \big)^2 \ind(X\in\C{D}) \Big]\nn\\
    &\geq& \chi^2\cdot\E\Big[ \big( \gbbY(\s(X)-X^\T \hY) - \nb(X) \big)^2 \ind(X\in\C{D}) \Big]\nn\\
    &\overset{\add\thecount\label{9}}{\geq}& \chi^2\cdot \E \Big[ \Var\big( \gbbY(\s(X)-X^\T \hY)\big| X \big) \ind(X\in\C{D}) \Big]\nn\\
    &=&  \chi^2\cdot \E \Big[ \Var\big( X^\T \hY \big| X \big) Q(X)^2 \ind(X\in\C{D}) \Big]\nn\\
    &=& \chi^2\cdot \E\Big[ \E\big[ \big(X^\T (\hY - \tY)\big)^2\big| X\big] Q(X)^2 \ind(X\in\C{D}) \Big]\nn\\
    &=& \chi^2\cdot \E\Big[ \Big( \big(Q(X)\ind(X\in\C{D})X\big)^\T(\hY - \tY) \Big)^2 \Big]\nn\\
    &\overset{\add\thecount\label{10}}{\geq}& \frac{\chi^2}{d}\cdot \E\Big[  \big(Q(X)\ind(X\in\C{D})X\big)^\T (Y - \hY) \Big]^2\nn\\
    &=& \frac{\chi^2}{d}\cdot \E\Big[  (X^\T Y-X^\T \hY) Q(X) \cdot \ind(X\in\C{D}) \Big]^2\nn\\
    &\overset{\add\thecount\label{11}}{\geq}& \frac{\chi^2}{d}\cdot \E\Big[  \gbbY(X^\T Y-X^\T \hY) \cdot \ind(X\in\C{D}) \Big]^2\nn\\
    &\overset{\add\thecount\label{12}}{\geq}& \frac{\chi^2}{d}\cdot \E\Big[  \gbY(X^\T Y-X^\T \hY) \cdot \ind(X\in\C{D}) \Big]^2\nn\\
    &\overset{\add\thecount\label{13}}{=}& \frac{\chi^2}{d}\cdot \E\Big[  \gbY(\s(X)-X^\T \hY) \cdot \ind(X\in\C{D}) \Big]^2,
\end{eqnarray}
where
\[
    \chi := \inf_{x\in\C{D}, y\in\Th} \frac{\gby(\s(x)-x^\T y)}{\gbby(\s(x)-x^\T y)},
\]
and
\[
    Q(x)^2 := \frac{\Var\big( \gbby(\s(x)- x^\T \hY)\big)}{\Var\big( x^\T \hY\big)}.
\]
In $\ref{9}$, we apply the fact that for any random variable $W$ with $\E[W^2]<\infty$ and constant $a$, there is 
\[
    \E[(W-a)^2] \geq \Var[W].
\]
In $\ref{10}$ we use the result in Lemma \ref{lemma:gbb}. In $\ref{11}$, we use the fact that
\begin{eqnarray}
    Q(x)^2 
    &=& \frac{\Var\big( \gbby(\s(x)-x^\T \hY)\big)}{\Var\big( \s(x) - x^\T \hY\big)} \nn\\
    &\geq& \frac{\E\big[ \gbby(\s(x)-x^\T \hY)\big]^2}{\E\big[ \s(x) - x^\T \hY\big]^2} \nn\\
    &=& \frac{\E\big[ \gbby(x^\T \a^{-1}(x)-x^\T \hY)\big]^2}{\E\big[ x^\T \a^{-1}(x) - x^\T \hY\big]^2}.
\end{eqnarray}
Step $\ref{12}$ follows from that $\gbbY\geq\gbY$, and the final step follows trivially from $\s(X) = X^\T \a^{-1}(X) = X^\T Y$.
Hence we can set $C_1=\frac{d}{\chi^2}$.\par
\begin{figure}
    \centering
    \includegraphics[width=200pt]{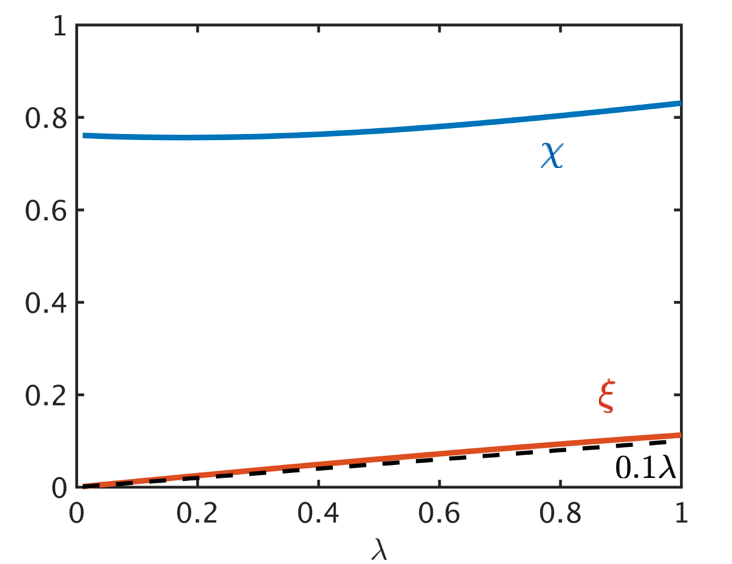}
    \caption{Constants $\chi$ and $\xi$ for $\b=2$.}
    \label{fig:constants2}
\end{figure}
Next we turn to constant $C_2$. We have that
\begin{eqnarray}
    \label{eq:proposition3-2}
    && \E\Big[ \big( \gbY(\s(X) - X^\T \hY) - \nb(X) \big)^2 \ind(X\in\bar{\C{D}}) \Big]\nn\\
    &\geq& \E\Big[ \big( \gbY(\s(X) - X^\T \hY) - \nb(X) \big)^2 \ind(\s(X) - X^\T \hY \leq \wbY)\ind(X\in\bar{\C{D}}) \Big]\nn\\ 
    &\geq& \E\Big[ \big( \gbY(\s(X) - X^\T \hY) - \nb(X) \big)^2 \ind(X^\T \hY \geq 0)\ind(\hY\in\a^{-1}(\bar{\C{D}}))\ind(X\in\bar{\C{D}}) \Big]\nn\\ 
    &\geq& \xi^2 \cdot \E\Big[ \ind(X^\T \hY\geq 0) \ind(\hY\in\a^{-1}(\bar{\C{D}}))\ind(X\in\bar{\C{D}})\Big]\nn\\
    &\overset{\add\thecount\label{14}}{\geq}& \frac{\xi^2}{2\eta}\cdot \E\Big[ \ind(\hY\in\a^{-1}(\bar{\C{D}})) \ind(X\in\bar{\C{D}})\Big]\nn\\
    &=& \frac{\xi^2}{2\eta}\cdot \E\Big[\ind(X\in\bar{\C{D}})\Big]^2\nn\\
    &\geq& \frac{\xi^2}{2\eta}\cdot  \E\Big[ \gbY(\s(X)-X^\T \hY) \cdot\ind(X\in\bar{\C{D}})\Big]^2,
\end{eqnarray}
with
\[
    \xi^2 := \inf_{x\in\bar{\C{D}}, y\in\Th} \big( \gby(\s(x) - x^\T y) - \nb(x)\big)^2,
\]
and $\ref{14}$ comes from Corollary \ref{cor:strongerTuran}. 
Thus we can set $C_2 = \frac{2\eta}{\xi^2}$.\par 
Finally, when $\b\geq 2$, we have that $\chi > \xi > 0.1\D$. Therefore
\begin{equation}
    \label{eq:DeltaIsNotOne}
    \G_t(\C{L}_\b, \pi^\TS) \leq 100\D^{-2} \eta.
\end{equation}
The values of the constants are plotted in Figure \ref{fig:constants2}. 
By combining \eqref{eq:DeltaIsNotOne} with \eqref{eq:cor:smallBeta}, we arrive at \eqref{eq:infoRatioBound}.
\newpage

\section{Proof of Proposition \ref{prop:NoSublinear}}
\label{sec:proof_of_prop_NoSublinear}
Suppose that for each $a\in\A$,
\[
    \E[R(a)|A^*=a] \geq 1-\d,
\]
and 
\[
    \max_{a'\neq a} \E[R(a')|A^*= a] \leq \d.
\]
Let $(\h{a}_1,\dots,\h{a}_t)$ be any deterministic action sequence up to stage $t$. Then conditioned on $A_1=\h{a}_1\dots A_t=\h{a}_t$, we have that $R_1,\dots,R_t$ are mutually independent. Hence 
\begin{eqnarray}
    \label{eq:ProbabilityOfZeros}
    &&\P\big( R_1=\dots=R_t=0\big| A_1=\h{a}_1,\dots,A_t=\h{a}_t \big)\nn\\
    &\geq& \P\big( R_1=\dots=R_t=0\big| A_1=\h{a}_1,\dots,A_t=\h{a}_t, A^*\notin\{\h{a}_1,\dots,\h{a}_t\} \big)\cdot \P\big(A^*\notin\{\h{a}_1,\dots,\h{a}_t\} \big)\nn\\
    &=& \left(\prod_{j=1}^t \P\big( R_j=0 \big| A_j=\h{a}_j, A^*\neq A_j \big) \right)\cdot \P\big(A^*\notin\{\h{a}_1,\dots,\h{a}_t\} \big)\nn\\
    &\geq& \big(1-\d\big)^t \cdot \left(1-\frac{t}{N}\right),
\end{eqnarray}
where in the final step we use the fact that the prior of $A^*$ is uniform. Let $\C{E}_t$ be the event $\{R_1=\dots=R_t=0\}$. Since \eqref{eq:ProbabilityOfZeros} holds for every action sequence, we have that for any policy $\pi$, 
\[
    \P( \C{E}_t ) \geq \big(1-\d\big)^t \cdot \left(1-\frac{t}{N}\right).
\]
Thus
\begin{eqnarray}
    \RM{BayesRegret}(t;\C{L},\pi)
    &=& \sum_{j=1}^t\E[R^* - R_j]\nn\\
    &=& t\E[R^*] - \sum_{j=1}^t\E[R_j|\C{E}_t] \P(\C{E}_t) - \sum_{j=1}^t\E[R_j|\bar{\C{E}}_t] \P(\bar{\C{E}}_t)\nn\\
    &\geq& (1-\d)t - \left[1-\big(1-\d\big)^t \cdot \left(1-\frac{t}{N}\right)\right] t\nn\\
    &=& \left[(1-\d)^{t} \frac{(N-t)}{N} - \d\right] t.
\end{eqnarray}
Let $\d=1/N$, we have that for $t\leq \frac{N}{2}-1$,
\[
    \RM{BayesRegret}(t;\C{L},\pi)\geq \frac{t}{2\sqrt{e}} \geq \frac{t}{4}.
\]
\newpage

\section{Upper Bounds of Fragility Dimension}
\label{sec:bounds_of_DOI}
In this section we give worst-case bounds of fragility dimension with respect to the problem dimension $d$. Let $\C{X}$ and $\C{Y}$ be two subsets of $\BF{B}_d$, and let $f^*:\C{Y}\mapsto\C{X}$ be such that
\[
    f^*(y)^\T y = \max_{x\in\C{X}} x^\T y,\quad \forall y\in\C{Y}.
\]
Further we define $\iota=\inf_{y\in \C{Y}} f^*(y)$. Here $\iota$ can be interpreted as the constant $\D$ in Assumption \ref{assp:main}. We will show that the worst-case bounds vary across the three regimes $\iota=1$, $\iota\in(0,1)$ and $\iota=0$.

\subsection{The Regime $\iota=1$}
\label{sec:bounds_of_DOI_linearRegime}
When $\iota=1$ since we are constraining $\C{X}$ and $\C{Y}$ to be contained in the unit ball, there must be that $f^*(y)=y$ for each $y\in\C{Y}$. Therefore $\eta(\C{X},\C{Y})$ is equal to the maximum integer $M$, such that there exists $\{y_1,\dots, y_M\}\subseteq \C{Y}$, with 
\[
    y_i^\T y_j < 0,\quad \forall i,j\in[M], i\neq j.
\]
The following lemma immediately implies that in this case $\eta(\C{X},\C{Y})\leq d+1$.
\begin{lemma}
\label{lem:linear}
In the $d$-dimensional Euclidean space, there exists at most $d+1$ different vectors, such that the inner-product between any pair of different vectors is negative.
\end{lemma}
\begin{proof}
Suppose that there exists a set $\C{X}$ which consists of $d+2$ different vectors $x_1, \dots x_{d+2}$, such that $x_i^\T x_j<0$ for any $1\leq i<j\leq d+2$. Let 
\[
U=
\begin{bmatrix}
x_1 &x_2 &\cdots &x_{d+2}
\end{bmatrix}.
\]
Then the nullspace of $U$ has dimension at least 2. Therefore we can find $z\in\mathbf{null}(U)\subset \BB{R}^{d+2}$, such that $z$ has at least one positive entry and one negative entry. Without loss of generality, we have that
\[
z_1x_1+z_2x_2+\dots + z_{k}x_k = -z_{\ell}x_\ell - z_{\ell+1}x_{\ell+1} -\dots -z_{d+2}x_{d+2},
\]
where $1<k<\ell<d+2$ and $z_1,\dots z_k>0$, $z_\ell\dots z_{d+2}<0$. However, this gives
\begin{eqnarray*}
&&\|z_1x_1+z_2x_2+\dots + z_{k}x_k\|_2^2\nn\\
&=& (z_1x_1+z_2x_2+\dots + z_{k}x_k)^\T(z_1x_1+z_2x_2+\dots + z_{k}x_k)\\
&=& -(z_1x_1+z_2x_2+\dots + z_{k}x_k)^\T(z_{\ell}x_\ell + z_{\ell+1}x_{\ell+1} +\dots +z_{d+2}x_{d+2})\\
&=& -\sum_{i=1}^k\sum_{j=\ell}^{d+2} z_iz_j x_i^\T x_j <0,
\end{eqnarray*}
which is a contradiction.
\end{proof}

\subsection{The Regime $\iota=0$}
\label{sec:bounds_of_DOI_0}
We show by an example for $d=3$ that when $\iota=0$, the fragility dimension can be arbitrarily large. Let $h, r\in(0,1)$ be constants to be determined later. Consider $\C{X}=\{x_1,\dots, x_N\}$ and $\C{Y}=\{y_1,\dots, y_N\}$ where
\[
    x_i = 
    \begin{pmatrix}
    r\cdot\cos\left(\frac{2\pi}{N}\cdot i\right), &
    r\cdot\sin\left(\frac{2\pi}{N}\cdot i\right), &
    \sqrt{1-r^2}
    \end{pmatrix},\quad i=1,\dots, N,
\]
and
\[
    y_i = 
    \begin{pmatrix}
    h\cdot\cos\left(\frac{2\pi}{N}\cdot i\right), &
    h\cdot\sin\left(\frac{2\pi}{N}\cdot i\right), &
    -\sqrt{1-h^2}
    \end{pmatrix},\quad i=1,\dots, N, 
\]
as is shown in Figure \ref{fig:counterexample}.
\begin{figure}
    \centering
    \includegraphics[width=150pt]{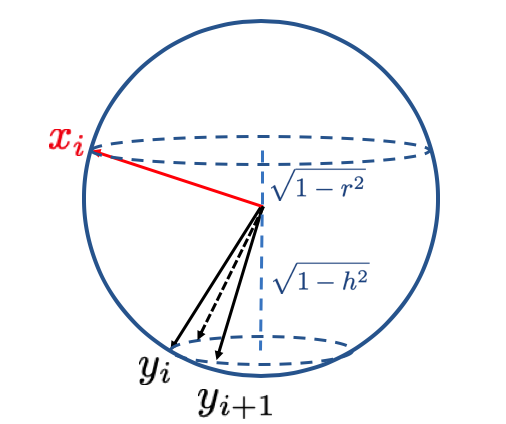}
    \caption{The example that $\eta(\C{X},\C{Y})$ is arbitrarily large.} 
    \label{fig:counterexample}
\end{figure}
We have that $f^*(y_i) = x_i$ and 
\[
    x_k^\T y_\ell = hr\cdot\cos\left(\frac{2\pi}{N}\cdot (k-\ell)\right) - \sqrt{(1-h^2)(1-r^2)}.
\]
To satisfy that $x_k^\T y_\ell < 0$ for all $k\neq \ell$, we only have to choose $h$ and $r$ such that
\[
    hr\cdot \cos\left(\frac{2\pi}{N}\right) - \sqrt{(1-h^2)(1-r^2)} < 0 < hr - \sqrt{(1-h^2)(1-r^2)}.
\]
This can be done by arbitrarily choosing $h$ and let $r = \sqrt{1-\g h^2}$ with
\[
    \frac{\cos^2\left(\frac{2\pi}{N}\right)}{1-\sin^2\left(\frac{2\pi}{N}\right)h^2} < \g < 1.
\]
Notice that $N$ can be arbitrarily large since $\iota=0$. Thus $\eta(\C{X},\C{Y})$ is unbounded.

\subsection{The Regime $\iota\in(0,1)$}
In this section we show that when $\iota\in(0,1)$, the worst-case fragility dimension grows exponentially with $d$. 
We first introduce the following result. We point readers to \cite{boroczky2004finite} for a detailed discussion.
\begin{fact}
    \label{fact:covering}
    For any $\e\in(0,1)$, there exists $\g>1$, such that for all integer $d\geq 3$, there exist $\g^d$ vectors in $\BF{S}_{d-1}$ such that the inner product of any two different vectors is at most $\e$.
\end{fact}
For any fixed $d$, let $u,v\in\left(0, \frac{\pi}{2}\right)$ and $\e>0$ be constants to be determined later. 
Let $z_1,\dots, z_N\in\BF{S}_{d-2}$ be such that
\[
    z_i^\T z_j < \e, \quad \forall j,k\in[N], j\neq k.
\]
Consider the pair of sets $\C{X},\C{Y}\subset\BF{S}_{d-1}$ defined by
\[
    \C{X}:=\left\{ x_i\right\}_{i=1}^N,\quad x_i = (\cos u, \sin u\cdot z_i),
\]
and 
\[
    \C{Y}:=\left\{ y_i \right\}_{i=1}^N,\quad y_i = (-\cos v, \sin v\cdot z_i).
\]
Thus we have
\[
    x_i^\T y_i = -\cos u\cos v + \sin u\sin v = -\cos(u+v),\quad i\in[N]
\]
and
\[
    x_j^\T y_k = -\cos u\cos v + z_j^\T z_k\sin u\sin v < -\cos(u+v) - (1-\e)\sin u\sin v,\quad j,k\in[N], j\neq k.
\]
There is obviously $f^*(y_i)=x_i$. 
In order to satisfy $\inf_{y\in\C{Y}} f^*(y)^\T y = \iota$, we only have to choose $u,v,\e$ such that  
\[
    \cos(u+v) \leq -\iota,
\]
and
\[
    \cos(u+v) + (1-\e)\sin u\sin v \geq 0.
\]
This can be done by setting
\[
    u = v = \frac{1}{2}\arccos(-\iota),\ \e = \frac{1-\iota}{1+\iota}.
\]
Since $\iota\in(0,1)$, we have that $\e\in(0,1)$. From Fact \ref{fact:covering}, there exists $\g>1$ such that $N\geq \g^{d-1}$.

\subsection{Removing Actions Could Make Problem Harder}
\label{sec:harder}
Let $\C{X}$ and $\C{Y}$ be the two sets given in the example in Appendix \ref{sec:bounds_of_DOI_0}. Let the parameter set be $\Th=\C{Y}$ and consider action sets $\A_1=\C{X}\cup\C{Y}$ and $\A_2=\C{X}$. Obviously $\A_2\subset\A_1$. However, we argue that the problem $\C{L}_1$ with action and parameter sets $(\A_1, \Th)$ is easier than the problem $\C{L}_2$ with sets $(\A_2, \Th)$.\par
In fact, from Lemma \ref{lem:linear}, we have that $\eta(\A_1,\Th)\leq 4$. However, the argument in Appendix \ref{sec:bounds_of_DOI_0} shows that $\eta(\A_1,\Th)=N$, where $N$ is the size of the parameter set. Therefore the regret of Thompson sampling on $\C{L}_1$ can be bounded by the result in Theorem \ref{thm:Clean}, which is independent of $\b$. However, to learn $\C{L}_2$ for a large $\b$, we almost have to try every action to find the optimal one. Therefore, somewhat surprisingly, reducing the size of the action set can actually make the problem harder.
\end{document}

%%%%%%%%%%%%%%%%%%%%%%%%% END OF DOCUMENT %%%%%%%%%%%%%%%%%%%%%%%%%